\documentclass[10pt,letterpaper]{article}
\pdfoutput=1
\usepackage{ifpdf}
\usepackage[utf8]{inputenc}
\usepackage{amsmath}
\usepackage{amsfonts}
\usepackage{amssymb}
\usepackage{dsfont}
\usepackage{float}
\usepackage{graphicx}
\usepackage{cancel}
\usepackage[margin=1in]{geometry}
\usepackage{afterpage}
\usepackage{xcolor}
\usepackage{subfig}
\usepackage{tabularx}
\newcolumntype{Y}{>{\centering\arraybackslash}X}
\usepackage{multirow}
\usepackage{hhline,colortbl}
\usepackage{fancyhdr}
\usepackage{amsthm}

\providecommand{\U}[1]{\protect\rule{.1in}{.1in}}
\newtheorem{lemma}{Lemma} 
\newtheorem{prop}{Proposition}
 
\usepackage[normalem]{ulem}

\author{Xiaotian Xu\thanks{Authors are with the Department of Mechanical Engineering,
		2181 Glenn L. Martin Hall, Building 088,
		University of Maryland,
		College Park, MD 20742, USA.
		Email: xxu0116@umd.edu, yancy@umd.edu.} 
	~and~Yancy Diaz-Mercado
}%
\title{\LARGE \bf
Swarm Herding: A Leader-Follower Framework For Multi-Robot Navigation}
\date{\vspace{-5ex}}

\begin{document}
\maketitle

	\begin{abstract} 
	
	A leader-follower framework is proposed for multi-robot navigation of large scale teams where the leader agents corral the follower agents. A group of leaders is modeled as a 2D deformable object where discrete masses (i.e., leader robots) are interconnected by springs and dampers. A time-varying domain is defined by the positions of leaders while the external forces induce deformations of the domain from its nominal configuration. The team of followers is performing coverage over the time-varying domain by employing a perspective transformation that maps between the nominal and deformed configurations. A decentralized control strategy is proposed where a leader only takes local sensing information and information about its neighbors (connected by virtual springs and dampers), and a follower only needs partial information about leaders and information about its Delaunay neighbors.
	
\end{abstract}%


\section{Introduction}

Multi-robot systems (MRS) have received a lot of attention for decades for their advantage of using coordinating efforts to achieve complex tasks and applications \cite{nagatani2011multirobot, huntsberger2003campout}. The navigation of a multi-robot system is one of the topics in the field which has excited researchers' interest in recent years. The behavior of a multi-robot system in navigation has been significantly inspired by some biological systems,  e.g., flocking of birds \cite{zou2018conservation}, swarming of bees \cite{funfhaus2018swarming}. It can be noticed that these biological systems possess common features to MRS such as consisting of a large number of individuals, each individual determining its motion based on local sensing information and communication only. 

The control of navigation of a multi-robot system contains multiple objectives: first, avoiding collision among robots and between a robot and evironmental obstacles using local information; second, maintaining certain level of cohesion between robots in the group; third, tracking the optimal path under the appropriate metric. Many algorithms for cooperative control of a group of robots have been developed, e.g., behavior-based controller \cite{1245651}, virtual structural approach \cite{du2017distributed}, leader-follower network \cite{hsu2004multiple}, potential field approach \cite{chaimowicz2005controlling}. However, most of these algorithms can only partially achieve the objectives above, thus these methods are usually modified and coupled together. In \cite{pimenta2013swarm}, the swarm is modeled as an incompressible fluid subject to external force generated by potential field, and the smoothed particle hydrodynamics method is used to deal with the inter-robot coordination. In \cite{roy2020geometric}, a region-based shape control strategy is proposed by combining the controller of updating a virtual-circle (ellipse) based on the sensing information of the environment and the controller of maintaining robots to stay inside the virtual-region and keep certain formation.

In this paper, to achieve the objectives mentioned above, a three-layer control strategy is proposed. Specifically, the first layer of the controller provides a global optimal reference path from a start position to a goal position in the environment. The reference path can either be obtained beforehand for the known environment, or be computed online with available sensing information in an unknown environment by using different existing off-the-shore path-planning algorithms (e.g., Dijkstra \cite{johnson1973note}, RRT \cite{lavalle1998rapidly}, road map method \cite{kavraki1998analysis}). One can employ any reliable path-planning algorithm to find the global optimal path under the desired metric and pass it as an input to the second layer of the controller.

The second layer of the control strategy coordinates the motion planning for a team of leader robots. A time-varying domain is defined by the leader robots. The leader robots will track the reference path provided by the first layer of the control strategy and avoid the obstacles based on the sensing information about the environment. The team of leader robots will be modeled as a deformable object, i.e., a mass-spring-damper (MSD) network. The path planning for 3D deformable objects are investigated in reduced configuration space \cite{mahoney2010deformable}. In \cite{anshelevich2000deformable}, 3D objects deform under manipulation constrains, and the path planner searches the roadmap given the initial and final configuraitons. In this paper, we consider a 2D deformable domain formed by leader robots and for which the control is decentralized since each leader robots only interact with its neighbors in the MSD network.

The third layer provides control for a team of follower robots which stay in the time-varing domain defined by leader robots and perform coverage over it \cite{xu2020multi}. This layer of controller allows the followers to distributed in the domain defined by the leaders, to coordinates with each other to avoid collision (in point particle case), and to efficiently capture the motion of leaders. As the leaders are avoiding the obstacles, the induced time-varying domain becomes non-convex which limits the applicability of standard coverage strategies. To address this problem, a perspective transformation is used to transform quadrilaterals defined by groups of four leaders into static rectangular ones. A follower only needs local information, i.e., partial information about the leaders and information about the followers which are in its Delaunay neighborhood, to figure out its motion, thus the third layer of the control is decentralized.

%

The organization of the paper is as follows. In Section \ref{Sec:MotionPlanOfLeaders}, we begin with modeling a group of leader robots as a mass-spring-damper system and design the external forces such that the motion of the leaders can be controlled. The control strategy of a team of follower robots is presented in Section \ref{Sec:Coverage}. In Section \ref{Simulation}, a simulation using the proposed control scheme is conducted. Conclusions are drawn in Section \ref{Conclusion}.

\section{Motion Planning of A team of Leaders}\label{Sec:MotionPlanOfLeaders} 
In this section, we will discuss the second layer of the proposed control strategy. The leader robots will coordinate with each other and interact with the environment to define a time-varying domain. This domain will be used to perform coverage control of follower robots in the third layer of the controller. The induced time-varying domain should preserve a minimum area for the followers to occupy, and we assume that there exists a reference path that allows for passage through the environment.

\subsection{Model of a team of leader robots}

In this paper, a team of leader robots is modeled as a mass-spring-damper network subject to external forces which are generated by multiple constraints. The leaders are considered as masses that are interconnected by virtual springs and dampers as shown in the Fig. \ref{Pic_MassSpringDamper}.

\begin{figure}[t!]\centering
	\includegraphics[width=0.6\linewidth,clip=true,trim=0 0 0 -10]{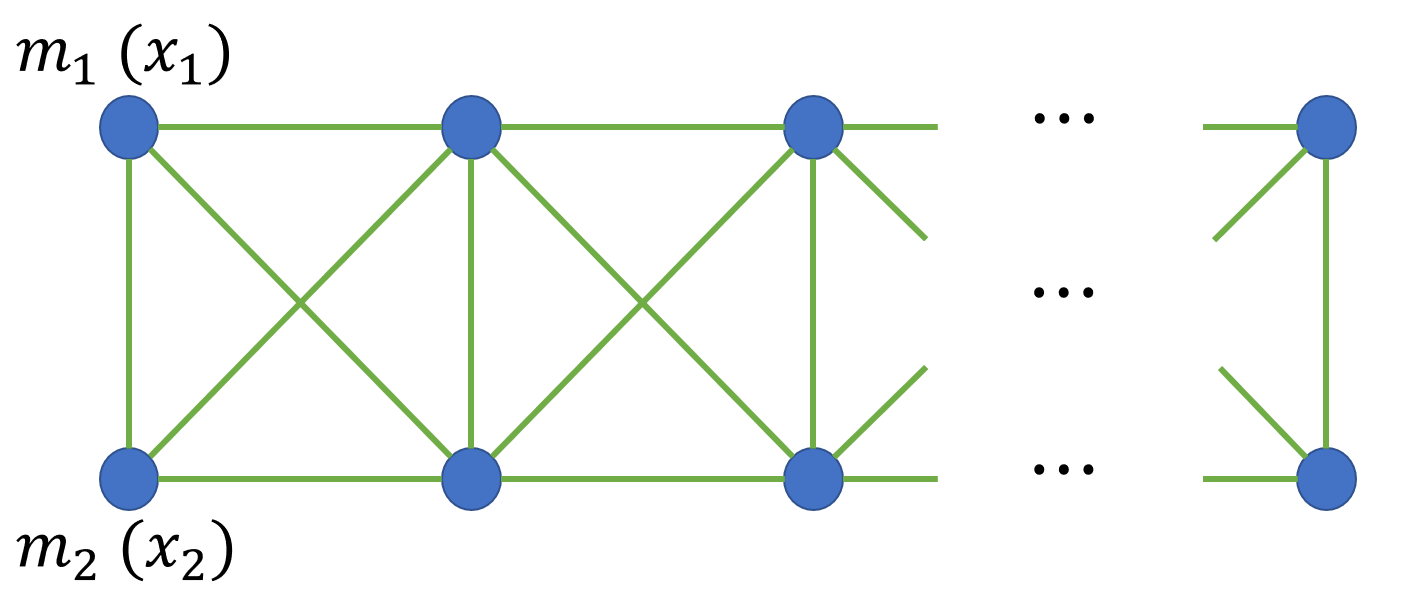}
	\caption{Mass-Spring-Damper Network}\label{Pic_MassSpringDamper}
\end{figure}

Let $x_k\in\mathbb{R},\,k = 1,\cdots,M$ ($M$ is an even number) denote the position of the leaders. Based on Newton's second law, the equation of motion is given by
\begin{align}\label{eq:MSD}
	m\ddot{x}_{k} = \sum_{l\in\mathcal{N}_{k}}g_{kl} + \sum_{l\in\mathcal{N}_{k}}f_{kl} + F_{k}
\end{align}
where $\mathcal{N}_{k}$ denotes a set which contains masses connected to mass $k$ (i.e., neighbor set of mass $k$), the spring force and the damping force on mass $i$ are $g_{kl} = - c(\dot{x}_k - \dot{x}_l)$ and $f_{kl} = -\kappa(\|x_k-x_l\|-l_{kl}^{o})\frac{x_k - x_l}{\|x_k-x_l\|}$ respectively, and the $F_k$ is the external forces applied on mass $k$. Moreover, $c$ denotes the damping constant, $\kappa$ denotes the stiffness, $l_{kl}^{o}$ represents the rest length of the spring between mass $k$ and $l$.

For such a system, the stability property can be described as the following proposition.
\begin{prop}[Stability of the MSD Network]\label{prop:1}
	The homogeneous version of the mass-spring-damper (MSD) system defined in eq. \eqref{eq:MSD} is asymptotically stable.
\end{prop}

\begin{proof}
	Without loss of generality, we consider a general mass-spring-damper network with $M$ masses. The homogeneous version of the motion equation can be rearranged,
	\begin{align*}
		\ddot{x}_{k} = \frac{1}{m}\sum_{l\in\mathcal{N}_{k}}c(\dot{x}_l - \dot{x}_k)
		+ \frac{1}{m}\sum_{l\in\mathcal{N}_{k}}\kappa(\|x_k-x_l\|-l_{kl}^{o})\frac{x_l - x_k}{\|x_k-x_l\|}
	\end{align*}
	Let us consider the topology of interaction for the MSD network to be given by an undirected and connected graph \cite{mesbahi2010graph}. Define the Laplacian matrix $L = DD^T$ and the weighted Laplacian matrix  \(L_{w} = DWD^{T}\), where $D$ denotes the incidence matrix which encodes the edge information, \(w_{kl} =\frac{\kappa(\|x_k-x_l\|-l_{kl}^o)}{m\|x_k-x_l\|} \) represents the weight on the edge between mass $k$ and $l$, and \(W = \text{diag}(w_{kl})\) is a diagonal matrix with $w_{kl}$ on its diagonal. The Laplacian matrix is positive semi-definite ($L\succeq 0$) and the weighted Laplacian matrix is symmetric ($L_{w} = L_{w}^T$). Then, by letting $x = [x_1,\cdots,x_M]^T$, the ensemble-level dynamics can be written as
	\begin{align*}
		\frac{d}{dt}\begin{bmatrix}
			x\\\dot{x}
		\end{bmatrix} =  \begin{bmatrix}
			0&I\\-L_{w}&-\frac{c}{m}L
		\end{bmatrix}\begin{bmatrix}
			x\\\dot{x}
		\end{bmatrix}
	\end{align*}
	Further define $\mathcal{E}_k =
	\sum_{l\in\mathcal{N}_{k}}\frac{\kappa}{2m}(\|x_k-x_l\|-l_{kl}^{o})^2$, and $\mathcal{E} =
	[\mathcal{E}_1, \mathcal{E}_2, \cdots,\mathcal{E}_M]^T$, and choose a candidate Lyapunov function to be 
	\begin{align*}
		V = \frac{1}{2}\mathds{1}^T \mathcal{E} + \frac{1}{2}\|\dot{x}\|^2,
	\end{align*}
	where $\mathds{1} = [1,\,1,\cdots,1]^T$, then
	\begin{align*}
		\dot{V} = \frac{1}{2}\mathds{1}^T \frac{\partial \mathcal{E}}{\partial x}\dot{x} + \dot{x}^T\ddot{x}
	\end{align*}
	where
	\begin{align*}
		\frac{\partial \mathcal{E}}{\partial x} = \begin{bmatrix}
			\frac{\partial \mathcal{E}_1}{\partial x_{1}}&\frac{\partial \mathcal{E}_1}{\partial x_{2}}&\cdots&\frac{\partial \mathcal{E}_1}{\partial x_{M}}\\
			\frac{\partial \mathcal{E}_2}{\partial x_{1}}&\ddots&&\vdots\\
			\vdots&&\ddots&\vdots\\
			\frac{\partial \mathcal{E}_M}{\partial x_{1}}&\cdots&\cdots&\frac{\partial \mathcal{E}_M}{\partial x_{M}}
		\end{bmatrix}.
	\end{align*}
	It can be found that 
	\begin{align*}
		\frac{\partial \mathcal{E}_k}{\partial x_{k}} = \sum_{l\in\mathcal{N}_{k}}\frac{\kappa(\|x_k-x_l\|-l_{kl}^o)}{m\|x_k-x_l\|}(x_k-x_l)^T
	\end{align*}
	\begin{align*}
		\frac{\partial \mathcal{E}_k}{\partial x_{l}} =\begin{cases}
			-\frac{\kappa(\|x_k-x_l\|-l_{kl}^o)}{m\|x_k-x_l\|}(x_k-x_l)^T,~~~&l\in\mathcal{N}_{k}\\
			0, &\text{otherwise}
		\end{cases}
	\end{align*}
	\begin{align*}
		\frac{\partial \mathcal{E}_l}{\partial x_{k}} =\begin{cases}
			-\frac{\kappa(\|x_l-x_k\|-l_{kl}^o)}{m\|x_l-x_k\|}(x_l-x_k)^T,~~~&k\in\mathcal{N}_{l}\\
			0, &\text{otherwise}
		\end{cases}
	\end{align*}
	and the $k$th column of $\mathds{1}^T \frac{\partial \mathcal{E}}{\partial x}$ yields
	\begin{align*}
		\mathds{1}^T\begin{bmatrix}
			\frac{\partial \mathcal{E}_1}{\partial x_{k}}\\	\frac{\partial \mathcal{E}_2}{\partial x_{k}}\\\vdots\\	\frac{\partial \mathcal{E}_M}{\partial x_{k}}
		\end{bmatrix} = 2\sum_{l\in\mathcal{N}_{k}}\frac{\kappa(\|x_k-x_l\|-l_{kl}^o)}{m\|x_k-x_l\|}(x_k-x_l)^T
	\end{align*}
	Then 
	\begin{align*}
		\mathds{1}^T \frac{\partial \mathcal{E}}{\partial x} &= 2\begin{bmatrix}
			\sum_{l\in\mathcal{N}_{1}}\frac{\kappa(\|x_1-x_l\|-l_{1l}^o)}{m\|x_1-x_l\|}(x_1-x_l)^T\\
			\sum_{l\in\mathcal{N}_{2}}\frac{\kappa(\|x_2-x_l\|-l_{2l}^o)}{m\|x_2-x_l\|}(x_2-x_l)^T\\
			\vdots\\
			\sum_{l\in\mathcal{N}_{M}}\frac{\kappa(\|x_M-x_l\|-l_{Nl}^o)}{m\|x_M-x_l\|}(x_M-x_l)^T
		\end{bmatrix}^T	\\
		& = 2(L_{w}x)^T.
	\end{align*}
	Hence,
	\begin{align*}
		\dot{V} &= \frac{1}{2}\mathds{1}^T \frac{\partial \mathcal{E}}{\partial x}\dot{x} + \dot{x}^T\ddot{x}\\
		& = x^TL_w^T\dot{x} + \dot{x}^T(-L_{w}x -\frac{c}{m}L\dot{x})\\
		& = -\frac{c}{m}\dot{x}^TL\dot{x} \leq 0
	\end{align*}
	Now define $\Omega_{\alpha} = \big\{x, \dot{x}\in\mathbb{R}^{N}\, \big|\,V(x,\dot{x})\leq \alpha ,\, \alpha >0\big\}$, and let $E = \big\{x, \dot{x}\in\Omega_{\alpha}\, \big|\,\dot{V}(x,\dot{x}) = 0\big\}= \big\{x, \dot{x}\,\big|\,x \in\mathbb{R}^{N}, \dot{x}\in \text{null}(L)=\text{span}\{\mathds{1}\}\big\}$, applying the invariance condition $\ddot{x} = 0$, we find that the largest invariant set $\Omega_{I} = \big\{x, \dot{x}\in\mathbb{R}^{N}\, \big|\,\|x_k-x_l\| = l_{kl}^o \,\forall k,l;\, \dot{x}\in \text{span}\{\mathds{1}\}\big\}$. By LaSalle's Invariance Principle \cite{lasalle1960some}, every solution to the system starts in $\Omega_{\alpha }$ will asymptotically approach to $\Omega_{I}$ as $t\to \infty$. Finally, let $\alpha\to \infty$, the set $\Omega_{I}$ is global asymptotically stable.
\end{proof}

From the proof above, without the external input force $F_k$, it can be noticed that given an initial condition the MSD system will asymptotically converge to the configuration with springs at their rest lengths, and all the masses have the same velocity. Now we can design the external forces such that the domain defined by leaders will move along the reference path, and the vertices of the domain (leader robots) will avoid the obstacles.

To satisfy the constraint that the domain defined by the leaders always sustains an area for the follower robots to occupy, we introduce the following lemma. 
\begin{lemma}[Constraint on deformation of MSD system]\label{lemma:DeformationOfMSD}
	Let the rest length of the vertical and horizontal springs be $l^{o}$, and the diagonal springs be $\sqrt{2}l^{o}$, then in the worst case, each triangle formed by three interconnected leaders must satisfies the triangle inequality when deforms subject to external forces. It can be found that if for each spring, the rate of elongation or compression does not exceed $\frac{2-\sqrt{2}}{2+\sqrt{2}}\times100\% \approx 17.2\%$, the induced domain will sustain certain amount of area, and each mesh defined by four interconnected leaders will be convex for all time.
\end{lemma}

Lemma \ref{lemma:DeformationOfMSD} imposes a restriction on the deformation of the mass-spring-damper network such that each mesh defined by 4 masses will always be a convex quadrilateral, and this property will be used in next section.

\subsection{Motion planning of the mass-spring-damper system}
Several constraints are imposed on the masses of the MSD network, the ``head" of the MSD network (i.e., the $m_1$ and $m_2$ shown in Fig. \ref{Pic_MassSpringDamper}) will take the responsibility of tracking the reference path and determining the orientation of the system, and all the masses (leader robots) will avoid the obstacles based on the sensing information.  

The external force applied to a mass is defined as
\begin{align}\label{eq:externalF}
	F_k = f_{\text{friction}}^{k} + f_{\text{tracking}}^{k} + f_{\text{orienting}}^{k}  + f_{\text{sensing}}^{k}
\end{align}
where $f_{\text{friction}}^{K}$ is a virtual friction force imposed on masses to dissipate kinetic energy stored in the system. $f_{\text{tracking}}^{k}$ is the force which drives the team of leaders to move along the reference path, the force $f_{\text{orienting}}^{k}$ adjusts the orientation of the system, and $f_{\text{sensing}}^{k}$ is the force generated based on sensing information that drives the $k$th leader robot to move away from the obstacles.

The sensing information about the environment obtained by the sensors installed on the leaders can be presented as $[	o_{k},\,d_{k} ]$, where $	o_{k}$ denotes the position of the nearest obstacle, and $d_{k}$ denotes the distance from the nearest obstacle detected in the $k$th leader's reference frame, then the external force which drives the leader robot away from the obstacle is defined as
\begin{align}\label{eq:sensingforce}
	f_{\text{sensing}}^{k} = -\frac{\kappa_{1}(x_{k}-o_{k} )}{\|x_{k}-o_{k}\|(d_{k}-\delta_{\text{sensing}})^2}
\end{align}
where $\kappa_{1}>0$ is a tuning parameter, and $\delta_{\text{sensing}}$ is the prescribed safe distance between a leader robot and the obstacles. Perceivably, $\|f_{\text{sensing}}^{k}\|\to \infty$ as $d_{k}\to\delta_{\text{sensing}}$.

Moreover, we assume that the first layer of our control scheme returns the information of global reference path in the form of $[	\Gamma,\,\dot{\Gamma} ]$, then $f_{\text{tracking}}^{k}$ is defined as
\begin{align}\label{eq:trackingforce}
	f_{\text{tracking}}^{k} = \begin{cases}
		\dot{c}_{x} = \dot{\Gamma} + \kappa_2(\Gamma-c_{x}),\,&k =1,\,2\\
		0 \,&\text{otherwise}
	\end{cases}
\end{align}
which has a feedforward term $\dot{\Gamma}$ and a feedback term $\kappa_2(\Gamma-c_{x})$ to drive the bisecting point of the positions of ``head" of the system (i.e., $c_{x} = \frac{1}{2}(x_1+x_2)$) to track the reference path, where $\kappa_2>0$ is a tuning parameter.

As we shown in Proposition \ref{prop:1}, even if the external force becomes zero, the system will keep moving at certain speed, thus friction is necessary to exist for dissipation of the kinetic energy stored in the system and can be designed as
\begin{align*}
	f_{\text{friction}}^{k}  = -\kappa_{3}\dot{x}_k
\end{align*} 
where $\kappa_{3}>0$ is a control parameter.

Besides, the ``head" of the MSD network (i.e., $m_1$ and $m_2$) also takes the responsibility to determine the orientation of the system, i.e., the line $\overline{x_1x_2}$ maintains orthogonal to the path $\Gamma$. Thus the force $f_{\text{orienting}}^{k}$ can be designed as
\begin{align}\label{eq:Orientingforce}
	f_{\text{orienting}}^{k} = \begin{cases}
		-\kappa_{4}\left((x_1-x_2)^T\dot{\Gamma}\right)\dot{\Gamma},\,&k = 1\\
		\kappa_{4}\left((x_1-x_2)^T\dot{\Gamma}\right)\dot{\Gamma},\,&k = 2\\
		0,\,&\text{otherwise}
	\end{cases}
\end{align}
where $\kappa_{4}>0$. The inner product of vector $x_1-x_2$ and the tangent of path $\dot{\Gamma}$ is greater than zero indicates the angle between the two vectors is less than $90^{\circ}$, then the the force $f_{\text{orienting}}^{1}$ will drive $m_1$ opposite to the direction of $\dot{\Gamma}$ until the inner product becomes zero.

It can be seen that by applying the designed external forces to the MSD network, the ``head" of network will steer and lead entire system, and the orientation information will be propagated through the network until it finally reaches the ``tail" of the network; the robots will also avoid the obstacles based on the sensing information. 

Now we can define a time-varying domain $\mathcal{S}(t)$ with vertices $\{x_{k},\,k = 1,\cdots,M\}$. The time-varying domain will evolve subject to the external force defined by \eqref{eq:externalF}. In the next section, we can design the control scheme for a team of followers.

\section{Coverage Control of A Team of Followers} \label{Sec:Coverage}

A coverage control algorithm naturally offers the ability to coordinate the motion of robots in a team. In this section, we will employ a coverage control law which is decentralized and can efficiently capture the evolution of the time-varying domain by extending ideas from our previous work \cite{xu2020multi}. The advantage of coverage control is that the individual controller for each robot in a team is synthesized such that the team can be controlled as a whole with time-varying densities\cite{SwarmBook} and time-varying domains\cite{xu2020multi}. For completeness, we begin this section with briefly introducing the coverage algorithm over convex time-varying domains.

\subsection{Coverage over time-varying convex domains}
To find the best configuration formed by the robotic team over a domain, we partition the domain into non-overlapping regions and let each region be taken care of by a robot. Formally, let
$p_{i}\in\mathcal{S}(t),i\in\{1,\ldots,N\}$ be the positions
of a group of follower robots in a convex domain $\mathcal{S}(t)$ defined by leaders, the coverage problem is dictated as finding the best configuration of robots $p(t)=\left[  p_{1}^{T}(t),\ldots,p_{N}^{T}(t)\right]^{T}$ to minimize the locational cost
\cite{locationalCost} which evaluates the coverage performance of $p$ over the domain
$\mathcal{S}$: 
\begin{equation} \label{eq:locationalCost}
	\mathcal{H}(p(t),t) = \sum_{i=1}^{n}
	\int_{V_{i}(p(t),t)} \|p_{i}(t)-q\|  ^{2} \phi(q,t) \,dq
\end{equation} 
where $\phi(q,t):\mathcal{S}(t)\times [0,\infty) \to\mathbb{R}_{>0}$ is a density
function that represents the relative importance of each point in the domain at time $t$, and $V_{i}(p,t) =
\left\lbrace q\in\mathcal{S}(t)~\middle\vert~\|p_{i}-q\|
\leq \|p_{j}-q\| ~\forall j\right\rbrace $ is a Voronoi tessellation.


It is known that a necessary condition to minimize the locational cost defined in \eqref{eq:locationalCost} is the centroidal Voronoi tessellation (CVT) configuration, i.e., $p_{i}(t) = C_{i}(p,t)\,\,\,\forall i$, where $C_{i}(p,t)$ is the cener of mass (centroid) of $i$th Voronoi cell. A control law is proposed for time-varying densities and time-varying domains to exponentially converge to a centroidal Voronoi configuration in \cite{SwarmBook, xu2020multi},
\begin{equation}
	\label{eq:TVD-C}\dot{p} = \left(  I - \frac{\partial C}{\partial p}\right)
	^{-1}\left(  K(C(p,t)-p) + \frac{\partial C}{\partial t}\right)
\end{equation} 
where the exponential convergence rate is tuned by the parameter $K>0$.

The control law given in \eqref{eq:TVD-C} can be approximated as
\begin{equation}
	\label{eq:TVD-D1}\dot{p} = \left(  I + \frac{\partial C}{\partial p}\right)
	\left(  K(C(p,t)-p) + \frac{\partial C}{\partial t}\right)
\end{equation} 
by truncating the Neumann series for the matrix inverse to make the control law decentralized.

The expressions of the centroids $C_{i}(p,t)$, and matrices  $\frac{\partial C}{\partial p}$ and $\frac{\partial C}{\partial t}$ in \eqref{eq:TVD-C} and
\eqref{eq:TVD-D1} can be found in \cite{xu2020multi}. The matrix  $\frac{\partial C}{\partial p}$ is a block matrix with sparsity structure encoding the adjacency 
information in the Voronoi tessellation \cite{SwarmBook}, and the matrix $\frac{\partial C}{\partial t}$ captures the evolution of time-varying densities and time-varying domains.

\begin{figure}[h]\centering
	\includegraphics[width=0.5\linewidth,clip=true,trim= 20 0 0 -10]{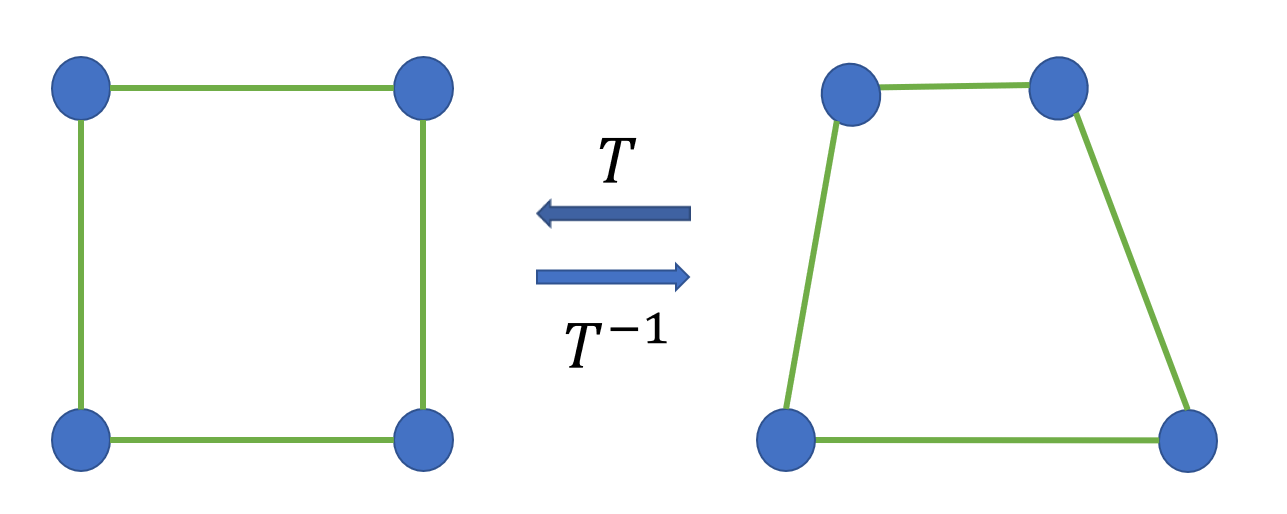}\\
	(i) Quadrilateral-to-quadrilateral projection
	\includegraphics[width=0.6\linewidth,clip=true,trim= 50 0 0 0]{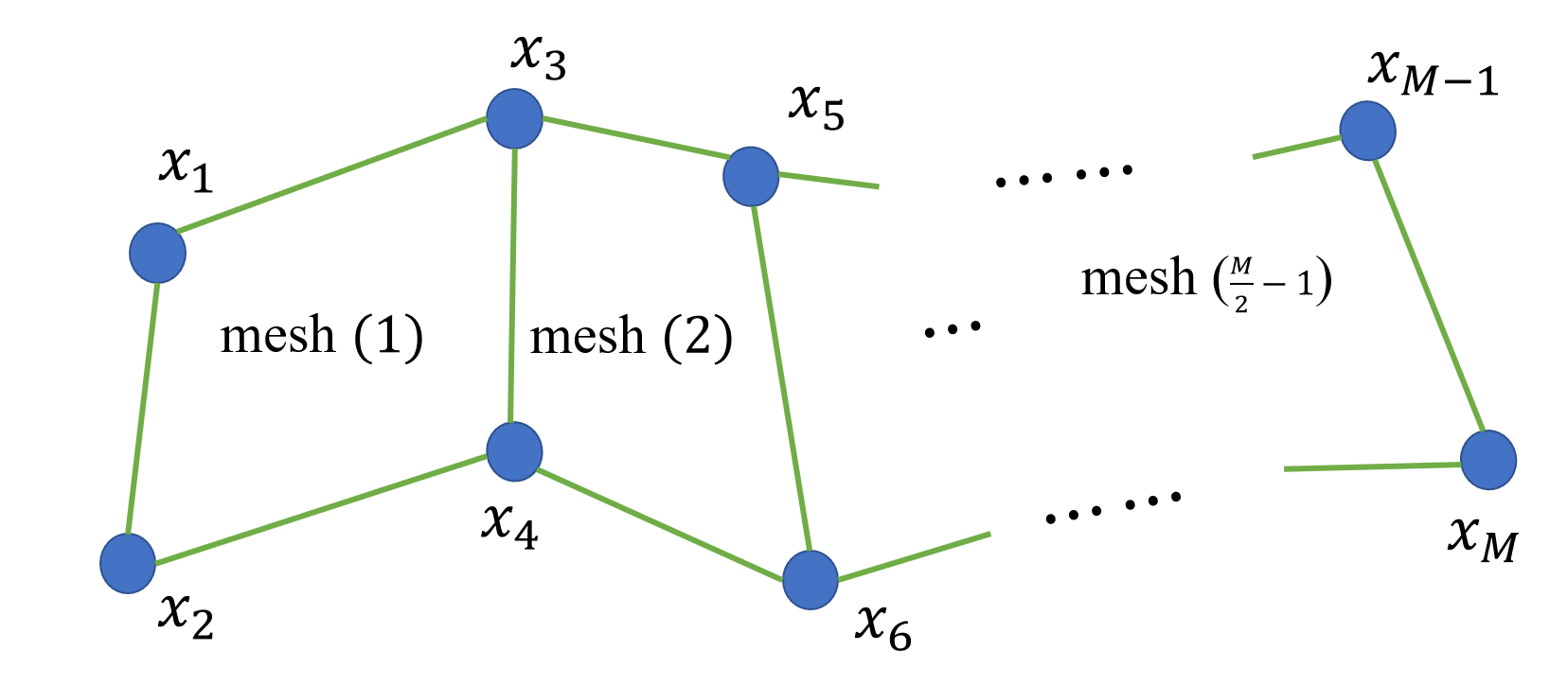}\\
	(ii) Time-varying domain defined by leaders
	\includegraphics[width=0.6\linewidth,clip=true,trim= 5 0 0 0]{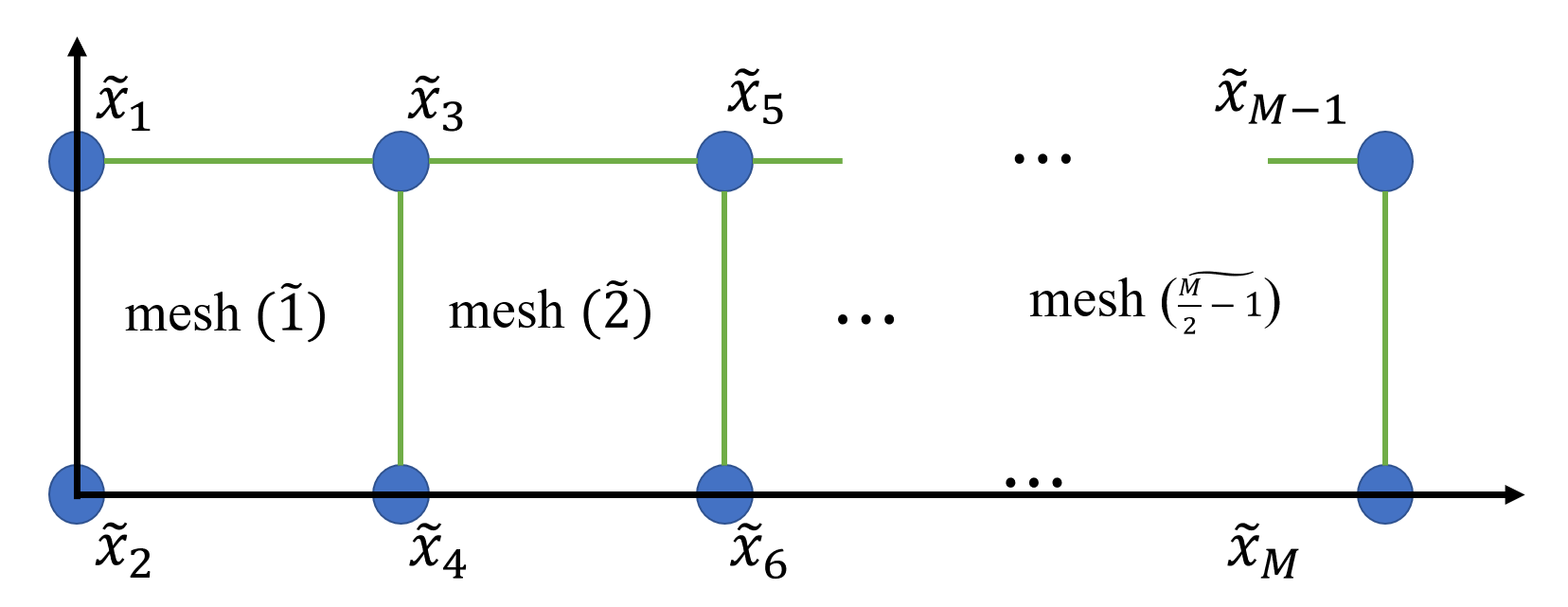}\\
	(iii) Resulting rectangular domain by projection
	\caption{Perspective Projection \label{Pic_Transformation}}
\end{figure}

The coverage control law in \eqref{eq:TVD-C} and \eqref{eq:TVD-D1} can only treat a time-varying convex domain. However, for a domain induced by a group of leader robots moving in a cluttered environment, it is likely that the domain will become non-convex. In previous work \cite{xu2020multiNoncnvx}, we dealt with a non-convex domain by using a diffeomorphism \cite{caicedo2008performing} to map points between the non-convex domain and its convex hull. In this paper, we will treat the problem in a different way that leverages the quadrilateral structure of the mesh formed by the leader agents.

\begin{figure*}[t] \centering
	\begin{tabular}{@{}>{\centering} m{0.245\linewidth} @{}>{\centering\arraybackslash }m{0.245\linewidth} @{}>{\centering\arraybackslash }m{0.245\linewidth}@{}>{\centering\arraybackslash }m{0.245\linewidth}@{}}
		\includegraphics[width=\linewidth,clip=true,trim= 0 0 0 0]{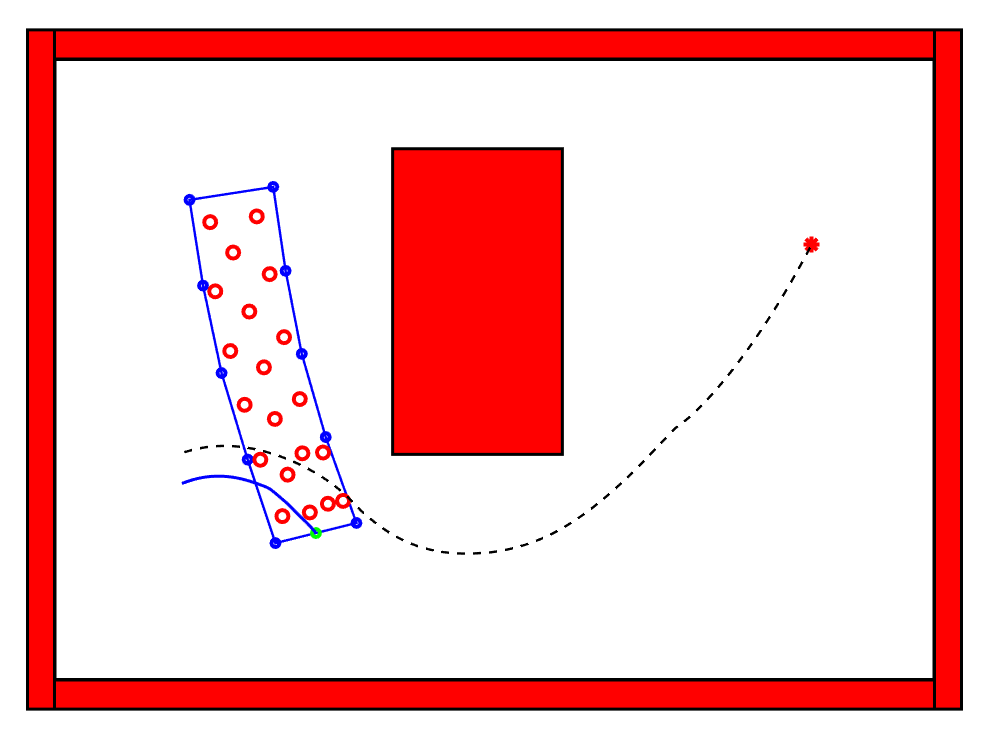} & 
		\includegraphics[width=\linewidth,clip=true,trim= 0 0 0 0]{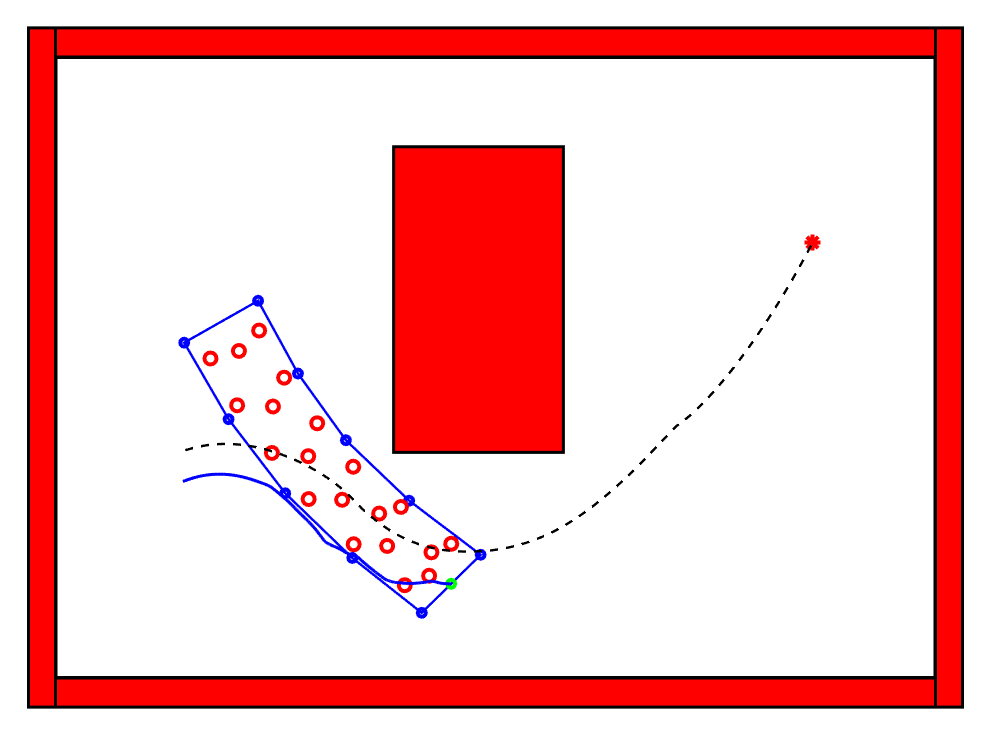} & 
		\includegraphics[width=\linewidth,clip=true,trim= 0 0 0 0]{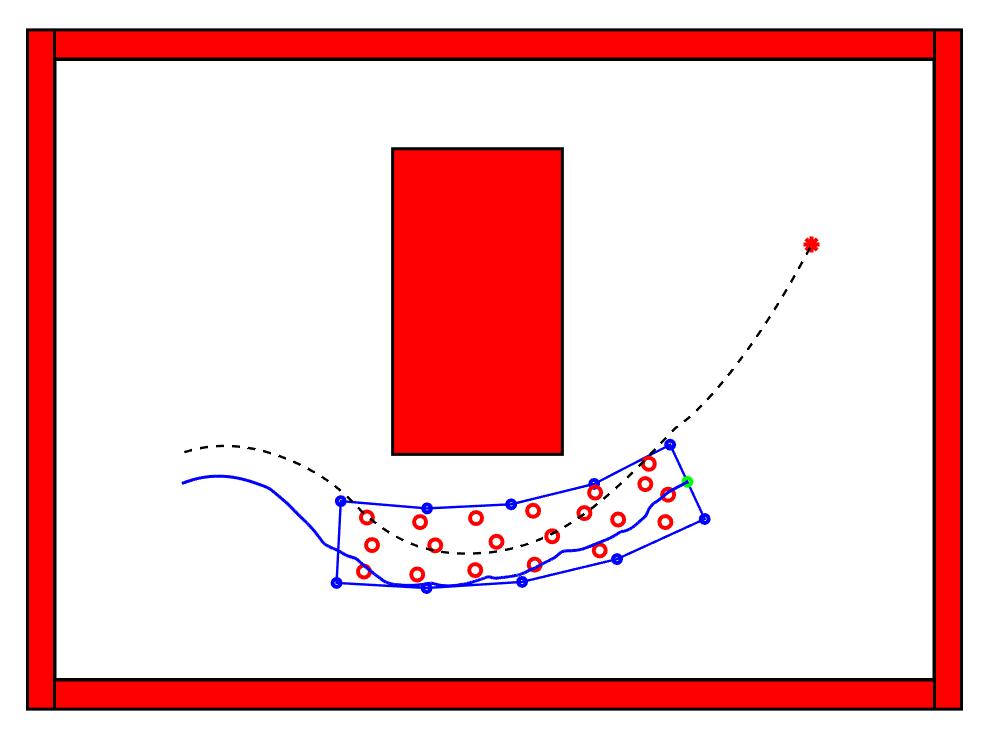}& 
		\includegraphics[width=\linewidth,clip=true,trim= 0 0 0 0]{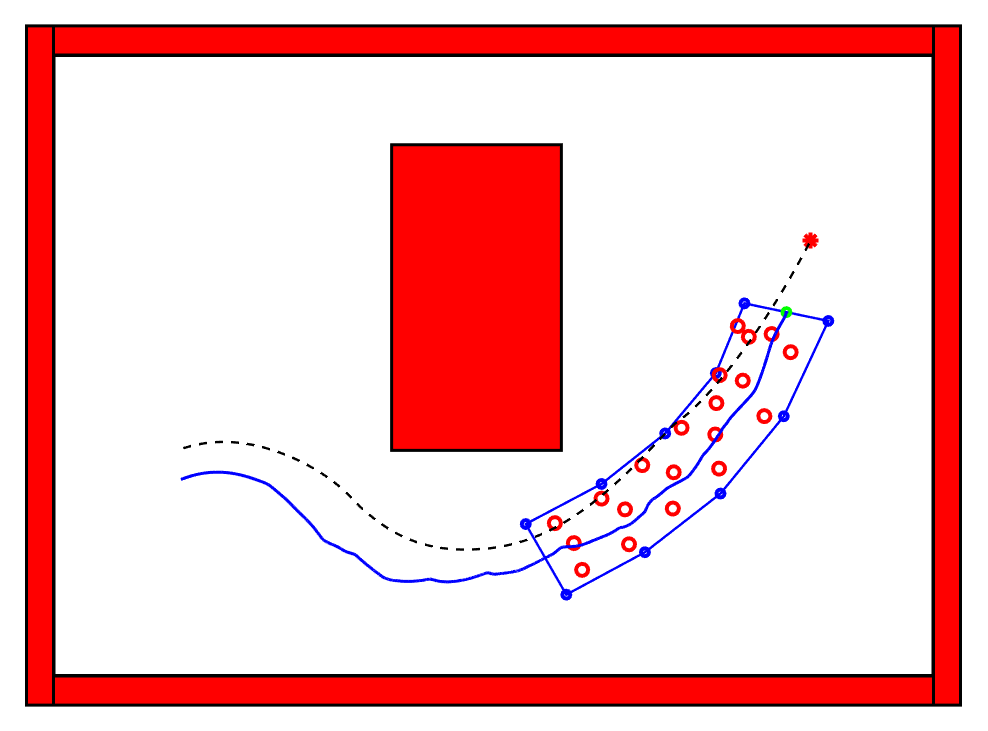}\\
		\includegraphics[width=\linewidth,clip=true,trim= 0 0 0 0]{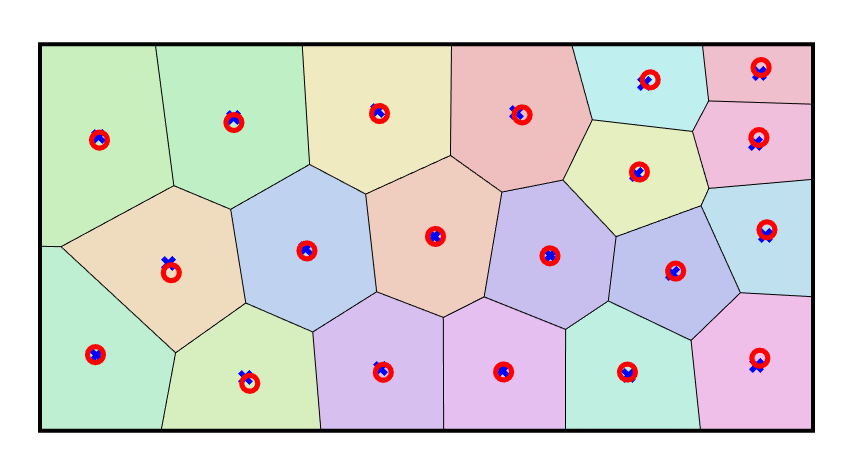} & \includegraphics[width=\linewidth,clip=true,trim= 0 0 0 0]{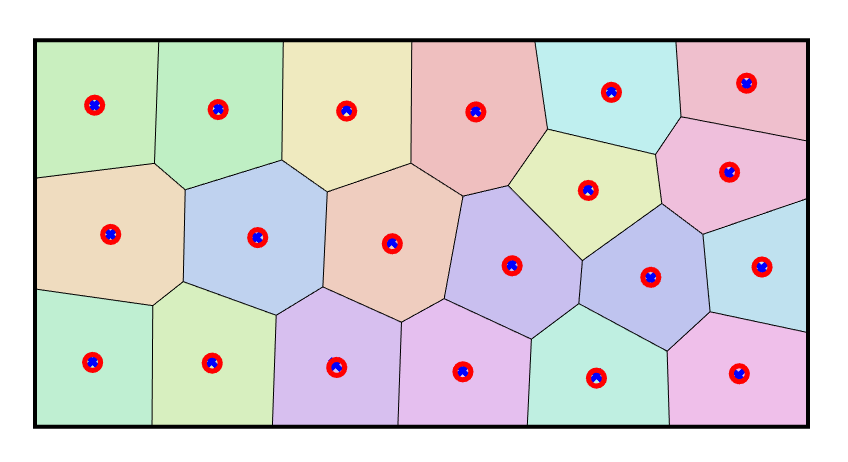} & \includegraphics[width=\linewidth,clip=true,trim= 0 0 0 0]{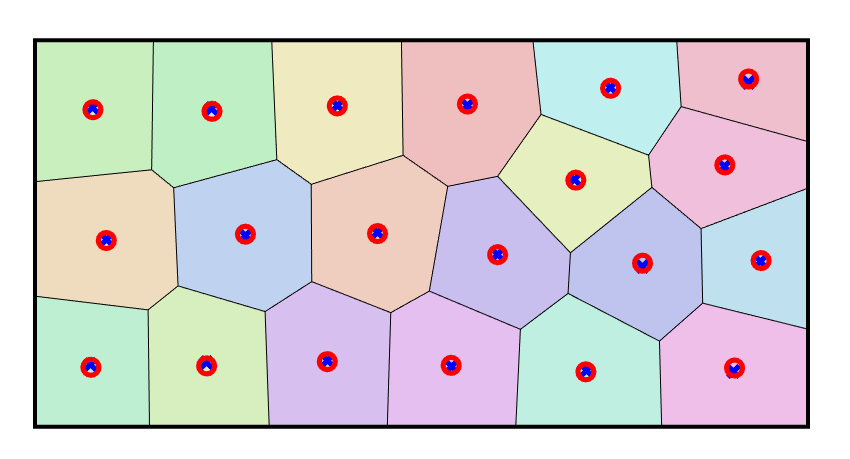}& 
		\includegraphics[width=\linewidth,clip=true,trim= 0 0 0 0]{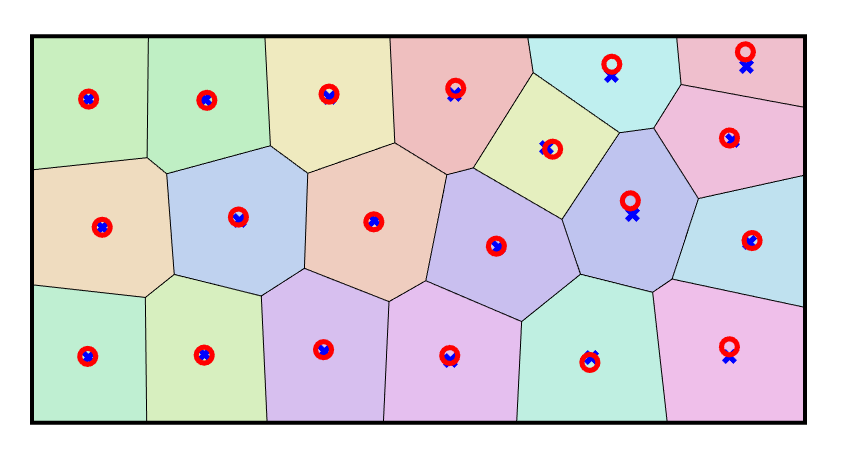}\\ 
		(i) $t = 3 s$ & (ii) $t = 6 s$ & (iii) $t = 14 s$ & (iv) $t = 18 s$
	\end{tabular}
	\caption{Ten leader robots (blue circles) are modeled as a mass-spring-damper system and used to define the time-varying domain. The domain becomes non-convex due to the external forces, i.e., tracking the reference path, avoiding the red obstacles and steering the orientation. Each mesh is transformed to a rectangle such that the MSD system is transfomed as a big static rectangle. The projection of twenty follower robots are performing coverage over the virtual rectangular domain, and the control signal is brought back to the real world by the inverse of the transformation. In (i), (ii), (iii) and (iv), show the simulation at 3$s$, 6$s$, 14$s$ and 18$s$ respectively. The black dashed curve represents the reference path, and the blue curve represents the actual trajectory of the "head" of the MSD system (mid point $c_x$ of $x_1$ and $x_2$). \label{Pic:Simulation}}
\end{figure*}

\begin{figure}[h]\centering
	\includegraphics[width=0.6\linewidth,clip=true,trim=0 0 0 10]{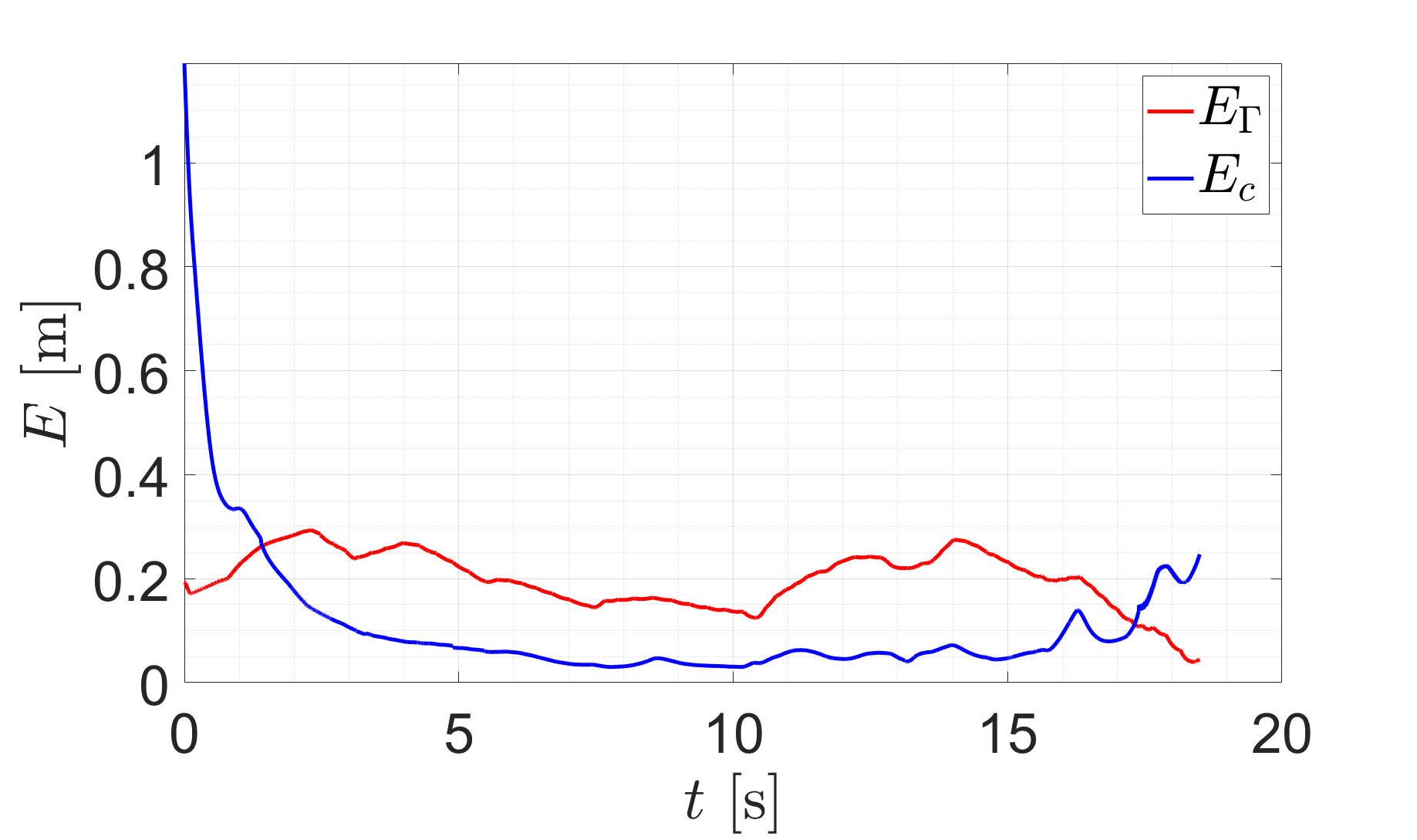}
	\caption{Tracking error and coverage aggregating error}\label{Pic_result}
\end{figure}

\subsection{Coverage over time-varying non-convex domains}

As a consequence of Lemma \ref{lemma:DeformationOfMSD}, a mesh in the mass-spring-damper network will always be a quadrilateral by restricting the elongation and compression of the springs, thus the time-varying domain defined by a team of leaders will be a composition of quadrilaterals. The strategy for performing coverage control of a group of followers over the composition of quadrilaterals will be to transform the successive quadrilaterals one by one such that a long rectangular region with prescribed height and length is obtained, as illustrated in Fig. \ref{Pic_Transformation}. In this way, the coverage control law \eqref{eq:TVD-C} and \eqref{eq:TVD-D1} can be directly utilized, and the resulting control signals can be transformed back into the actual domain.

Without loss of generality, let the $M$ leaders ($M$ being an even number) be labeled as illustrated in Fig. \ref{Pic_Transformation}(ii), then there are $\frac{M}{2}-1$ meshes in the network, and the $h$th mesh (denoted as mesh$(h)$) will be defined by vertices $x_{2h-1},\,x_{2h},\,x_{2h+2}$ and $x_{2h+1}, \, h = 1,\cdots, \frac{M}{2}-1$ in counterclockwise order. Given these vertices and by using a perspective projection as in \cite{criminisi1999plane}, the mesh$(h)$ can be transformed to an arbitrary convex quadrilateral mesh$(\widetilde{h})$ of vertices with prescribed positions $\widetilde{x}_{2h-1},\,\widetilde{x}_{2h},\,\widetilde{x}_{2h+2}$ and $\widetilde{x}_{2h+1}$, i.e., $T_h:\text{mesh}(h)\to\text{mesh}(\widetilde{h})$.

The position of $i$th follower robot $p_{i}$ in the actual domain will be transformed by $T_h$ if $p_{i}\in \text{mesh}(h)$, i.e., $\widetilde{p}_{i} = T_h(p_{i})$. The transformed followers $\{\widetilde{p}_{i}\}_{i=1}^N$ will be performing coverage over the resulting long rectangular domain $\widetilde{\mathcal{S}}(t)$ defined by vertices $\{\widetilde{x}_{k}\}_{k=1}^{M}$.

For $\widetilde{p}_{i} = T_h(p_{i})$, by the chain rule we can obtain that 
\begin{align*}
	\dot{\widetilde{p}}_{i} = \frac{\partial T_h(p_{i})}{\partial p_i} \dot{p}_i+ \frac{\partial T_h(p_{i})}{\partial t}
\end{align*}
which leads to
\begin{align}\label{eq:Transformed_ControlLaw}
	\dot{p}_i = \left(\frac{\partial T_h(p_{i})}{\partial p_i}\right)^{-1}\left(\dot{\widetilde{p}}_{i} - \frac{\partial T_h(p_{i})}{\partial t}\right)
\end{align}
where from \eqref{eq:TVD-C} and \eqref{eq:TVD-D1},
\begin{equation*}\label{eq:TVD-C_transformed}
	\dot{\widetilde{p}}_i = \left(  I - \frac{\partial \widetilde{C}_i}{\partial \widetilde{p}_i}\right)
	^{-1}\left(  K(\widetilde{C}_i-\widetilde{p}_i) + \frac{\partial \widetilde{C}_i}{\partial t}\right).
\end{equation*}
we can find that $\frac{\partial \widetilde{C_{i}}}{\partial t} $ is zero in this case since the time-dependency of the domain is captured by the defined transformation $T_h$ when the result virtual rectangular domain and the density are static.

\section{Multi-Robot Simulation}
\label{Simulation}

In previous sections, a time-varying domain is defined by a group of leader robots, and a coverage control law \eqref{eq:Transformed_ControlLaw} is developed to distribute a team of follower robots in a non-convex environment. The proposed control strategy is validated through simulation in this section.

In the simulation, we employ potential field method to obtain a reference path in the first layer of the controller. Then ten leader robots are modeled as masses interconnected by virtual springs and dampers subject to external forces. Each leader is equipped with sensors which can generate forces near obstacles, the first two leaders are tracking the reference path and determine the orientation of the MSD system. A time-varying non-convex domain is defined by considering the positions of leaders as vertices. For each quadrilateral mesh defined by 4 robots, we develop a perspective transformation to transform it into a static rectangle, such that the successive meshes are transformed into a large static rectangle as shown in Fig. \ref{Pic:Simulation}. A team of ten follower robots is also transformed into the virtual rectangular domain and commanded to perform coverage. The velocity commands generated in the virtual domain can then be mapped back by using the inverse transformation.

We employ two metrics to evaluate the performance of the leaders and followers, i.e., the reference path tracking error $E_{\Gamma}$ and aggregate coverage error $E_{c}$. The tracking error captures the minimum distance from the ``head" of the MSD system (mid point $c_x$ between $x_1$ and $x_2$) to the path, and the aggregate coverage error captures the error between followers and centroids for the corresponding Voronoi cells in the virtual domain. The result is shown in Fig. \ref{Pic_result}. As we can see from Fig. \ref{Pic:Simulation} and Fig. \ref{Pic_result}, the leaders are able to track the reference path while avoiding the obstacle. If there is no obstacles in the environment, the reference path tracking error will asymptotically approach to zero due to the tracking force designed by \eqref{eq:trackingforce}; but in a cluttered environment, the tracking error is influenced by several factors, e.g., the size of the time-varying domain, and the obstacles which locate on the both sides of the reference path. The followers in the real world move correspondingly to maintain a centroidal Voronoi tessellation configuration in the virtual domain.

\section{Conclusion}\label{Conclusion} 

In this paper, we present a leader-follower framework for navigation of a multi-robot system. The team of leaders are modeled a MSD system which navigates and interacts with the environment. A team of followers coordinates with each other and distribute optimally over the domain defined by leaders in a transformed virtual convex domain, and this results in near-optimal distribution of the followers in the real world when the deformation of the domain is slight. The leaders are able to track the reference path, avoid the obstacles and sustain a minimum amount of area required for the followers to distribute over. The followers are capable of occupying the domain defined by leaders and capture the motion of the leaders in a decentralized fashion.

\bibliographystyle{unsrt} 
\bibliography{ICRA2021_arXive}
\end{document}